\documentclass[11pt]{article}
\usepackage{amsmath,amssymb,amsfonts,amsthm,epsfig}
\usepackage[usenames,dvipsnames]{xcolor}
\usepackage{bm,xspace}
\usepackage{tcolorbox}
\usepackage{cancel}
\usepackage{fullpage}
\usepackage{liyang}
\usepackage{framed}
\usepackage{verbatim}
\usepackage{enumitem}
\usepackage{array}
\usepackage{multirow}
\usepackage{afterpage}
\usepackage{mathrsfs}
\usepackage{pifont} 
\usepackage{chngpage}
\usepackage[normalem]{ulem}
\usepackage{boxedminipage}
\usepackage{caption}
\usepackage{forest}

\usepackage{algorithm}
\usepackage{algorithmicx}
\usepackage{algpseudocode}

\usepackage{pgfplots}
\pgfplotsset{width=8cm,compat=newest}

\def\colorful{0}

\ifnum\colorful=1

\fi
\ifnum\colorful=0

\fi

\newcommand{\TV}{\dist_{\mathrm{TV}}}

\newcommand{\BuildDT}{\textsc{BuildDT}}

\newcommand{\ArbDist}{\mathcal{E}}
\newcommand{\InfEst}{\textsc{InfEst}}
\newcommand{\Lift}{\textsc{LiftLearner}}

\newlist{enumprop}{enumerate}{1} 
\setlist[enumprop]{label=\arabic*.,ref=\theproposition.\arabic*}
\crefalias{enumpropi}{proposition}

\makeatletter
\newtheorem*{rep@theorem}{\rep@title}
\newcommand{\newreptheorem}[2]{
\newenvironment{rep#1}[1]{
 \def\rep@title{#2 \ref{##1}}
 \begin{rep@theorem}\itshape}
 {\end{rep@theorem}}}
\makeatother

\newreptheorem{theorem}{Theorem}

\begin{document}
\title{Lifting uniform learners via distributional decomposition  
}

\author{}
\author{Guy Blanc \vspace{8pt} \\ \hspace{-5pt}{\sl Stanford}
\and \hspace{0pt} Jane Lange \vspace{8pt} \\ \hspace{-4pt}  {\sl MIT}
\and Ali Malik \vspace{8pt}\\ \hspace{-8pt} {\sl Stanford}
\and Li-Yang Tan \vspace{8pt} \\ \hspace{-8pt} {\sl Stanford}}  

\date{\vspace{15pt}\small{\today}}

\maketitle



\begin{abstract}
    We show how any PAC learning algorithm that works under the uniform distribution can be transformed, in a blackbox fashion, into one that works under an  arbitrary and unknown distribution~$\mathcal{D}$.  The efficiency of our transformation scales with the inherent complexity of~$\mathcal{D}$, running in $\poly(n, (md)^d)$ time for distributions over $\bits^n$ whose pmfs are computed by depth-$d$ decision trees, where $m$ is the sample complexity of the original algorithm.   For monotone distributions our transformation uses only samples from~$\mathcal{D}$, and for general ones it uses subcube conditioning samples.

A key technical ingredient is an algorithm which, given the aforementioned access to $\mathcal{D}$, produces an {\sl optimal} decision tree decomposition of $\mathcal{D}$: an approximation of $\mathcal{D}$ as a mixture of uniform distributions over disjoint subcubes.   With this decomposition in hand, we run the uniform-distribution learner on each subcube and combine the hypotheses using the decision tree. This algorithmic decomposition lemma also yields new algorithms for learning decision tree distributions  with runtimes that exponentially improve on the prior state of the art---results of independent interest in distribution learning.

\end{abstract}


\thispagestyle{empty}
\newpage 
\setcounter{page}{1}

\newcommand{\depth}{\mathrm{depth}}

\section{Introduction}

A major strand of research in learning theory concerns the learning of an unknown target  function $f : \bits^n\to\zo$ with respect to an unknown source distribution $\mathcal{D}$.  This line of work has both motivated and benefited from rich connections to the study of function complexity: underlying every new learning algorithm are new structural insights about the concept class of functions that $f$ belongs to.  Our work is motivated by the question of whether structural properties of the {\sl source distribution $\mathcal{D}$} can be similarly leveraged for the design of efficient algorithms.  




\paragraph{Distribution-free learning.} In Valiant's original distribution-free PAC learning model~\cite{Val84}, the target function $f$ is promised to belong to a {\sl known}, and often ``computationally simple" concept class $\mathscr{C}$, whereas no assumptions are made about the distribution $\mathcal{D}$.

Unfortunately, efficient algorithms have proven hard to come by even for simple concept classes.  Consider, for instance, the class of polynomial-size DNF formulas---the motivating example that Valiant used throughout his paper. Despite decades of effort, the current fastest algorithm for this class runs in exponential time, $\exp(\tilde{O}(n^{1/3}))$~\cite{KS04}.  For other classes, even formal hardness results have been established.  For example, under standard cryptographic assumptions it has been shown that there are no efficient algorithms for learning intersections of polynomially many halfspaces~\cite{KS09}.  

A major source of difficulty in the design of efficient learning algorithms in this model stems from the absence of assumptions about $\mathcal{D}$.  Even if the concept class $\mathscr{C}$ is assumed to be simple, this can be negated by the fact that the underlying distribution $\mathcal{D}$ is arbitrarily complex. Indeed, this is precisely the idea underlying existing hardness results for distribution-free learning. 

\paragraph{Distribution-specific learning.} This leads us to the {\sl distribution-specific} variant of PAC learning where $\mathcal{D}$ is known rather than arbitrary \cite{BenedekItai91,Natarajan92}.  For the domain
$\bits^n$, the most popular assumption takes $\mathcal{D}$ to be the uniform distribution---an assumption that opens the door to many more positive results.  For example, 
there is a quasipolynomial-time algorithm for learning polynomial-size DNF formulas under the uniform distribution~\cite{Ver90}; if the learner is additionally given query access to $f$, there are even polynomial-time algorithms~\cite{Jac97,GKK08}. {{ Similarly, there is a  quasipolynomial-time algorithm for learning intersections of polynomially many halfspaces under the uniform distribution~\cite{KOS04}---a sharp contrast to the distribution-free setting where even learning intersections of {\sl two} halfspaces in subexponential time is an longstanding open problem~\cite{KS08,She13}.  Indeed, there is by now a suite of powerful techniques for the design of uniform-distribution learning algorithms, notably  ``meta-algorithms" for learning an arbitrary concept class $\mathscr{C}$ with performance guarantees that scale with its Fourier properties.  Famous examples include the algorithms for learning (approximately) low-degree functions~\cite{LMN93} and (approximately) sparse functions~\cite{KM93}, as well as their agnostic variants~\cite{KKMS08,GKK08}; see Chapter \S3 of~\cite{ODBook} for an in-depth treatment.}}

While this setting has proved to be fertile grounds for the development of sophisticated techniques and fast algorithms, it is admittedly a stylized theoretical model, since real-world data distributions often exhibit correlations and are not uniform.  For this reason, uniform-distribution learners have thus far been limited in their practical relevance.  

\subsection{This work}  These two settings, distribution-free and uniform-distribution learning, correspond to two extremes in terms of distributional assumptions.  Distribution-free algorithms make no assumptions about the distribution, but the design of such algorithms that are computationally efficient has proven correspondingly challenging.  On the other hand, we have a wealth of techniques and positive results in the uniform-distribution setting, but the guarantees of such algorithms only hold under a strong, often unrealistic assumption on the data distribution. 
Due to these vast differences, research in these two settings has mostly proceeded independently with few known connections. 

The motivation for our work is the search for fruitful middle grounds between these two extremes, where efficient algorithms can be obtained for expressive concept classes and where their guarantees hold for broad classes of distributions.  More generally, we ask:
\begin{quote}
    \emph{Can we interpolate between the two extremes via notions of {\sl distribution complexity}, and design learning algorithms whose efficiency scale with the inherent complexity of $\mathcal{D}$?}
\end{quote} 
Relatedly, can the large body of existing results and techniques for learning under the uniform distribution be lifted, ideally in a blackbox fashion, to non-uniform but nevertheless ``simple" distributions? 
\section{Our results} 

We provide affirmative answers to these questions via a natural notion of distribution complexity: 

\begin{definition}[Decision tree complexity of $\mathcal{D}$] 
\label{def:DT} 
The {\sl decision tree complexity} of a distribution $\mathcal{D}$ over $\bits^n$, denoted $\depth(\mathcal{D})$, is the smallest integer $d$ such that its probability mass function (pmf) can be computed by a depth-$d$ decision tree. 
\end{definition} 

Decision tree complexity interpolates between the uniform distribution on one extreme (depth $0$) and arbitrary distributions on the other (depth $n$).  

This notion, {{first considered by Feldman, O'Donnell, and Servedio~\cite{FOS08},}} generalizes other well-studied notions of complexity of high-dimensional discrete distributions and is itself a special cases of others.  A {\sl $d$-junta distribution} $\mathcal{D}$, introduced by Aliakbarpour, Blais, and Rubinfeld~\cite{ABR16}, is one for which there exists a set of $d$ coordinates $J \sse [n]$ such that for every assignment $\rho \in \bits^J$, the conditional distribution $\mathcal{D}_{J \leftarrow\rho}$ is  uniform.  Every $d$-junta distribution has decision tree complexity $d$, but a depth-$d$ decision tree distribution can have junta complexity as large as $2^d$.  On the other hand,  decision tree distributions are a special case of mixtures of subcubes~\cite{CM19}, since a depth-$d$ decision tree induces a partition of $\bits^n$ into $2^d$ many disjoint subcubes. Mixtures of subcubes are in turn a special case of mixtures of product distributions over discrete domains~\cite{FM99,Cry99,CGG01,FOS08}.

\subsection{First main result: Lifting uniform-distribution learners}
 
 Our algorithms and analyses will use the notion of {\sl influence} of variables.  This notion is most commonly applied to boolean-valued functions, and in this work we extend its study to the pmfs of distributions: 
 
 \begin{definition}[Influence of variables on distributions]
 \label{def:inf} 
Let $\mathcal{D}$ be a distribution over $\bits^n$ and  $f_\mathcal{D}(x) = 2^n \cdot \mathcal{D}(x)$ be its pmf scaled up by the domain size.\footnote{This $2^n$ normalisation factor makes the average value of $f_\mathcal{D}$ exactly $1$, lending to a cleaner analysis.}  The {\sl influence} of a coordinate $i\in [n]$ on a distribution $\mathcal{D}$ over $\bits^n$ is the quantity 
 \[ 
 \Inf_i(f_\mathcal{D}) \coloneqq  \Ex_{\bx \sim \mcU^n}\big[ |f_\mathcal{D}(\bx)-f_\mathcal{D}(\bx^{\sim i})| \big], 
 \]
 where $\mcU^n$ denotes the uniform distribution over $\bits^n$ and $\bx^{\sim i}$ denotes $\bx$ with its $i$-th coordinate rerandomized.  The {\sl total influence} of $\mathcal{D}$ is the quantity $\Inf(f_\mathcal{D}) \coloneqq \sum_{i=1}^n \Inf_i(f_\mathcal{D}).$
 \end{definition}

With~\Cref{def:DT,def:inf} in hand we can now state our first main result.  For clarity, we first state it assuming a unit-time oracle that computes the influences of variables: 

\begin{theorem}[Lifting uniform-distribution learners; see~\Cref{thm:lift-formal} for the formal version]
\label{thm:lift} 
    For any concept class $\mathscr{C}$ of functions $f: \bits^n \to \zo$ closed under restrictions, assuming a unit-time influence oracle, if there is an algorithm for learning $\mathscr{C}$ under the uniform distribution to accuracy $\eps$ with sample complexity $m$ and running time $\poly(n,m)$, there is an algorithm for learning $\mathscr{C}$ under depth-$d$ decision tree distributions with sample complexity and running time
    \begin{equation*}
        M = \poly(n) \cdot \left(\frac{dm}{\eps}\right)^{O(d)}.
    \end{equation*}
\end{theorem}

{As an example setting of parameters,~\Cref{thm:lift} lifts a quasipolynomial-time uniform-distribution algorithm into one that still runs in quasipolynomial time, but now succeeds under any distribution with decision tree complexity $\polylog(n)$.  We note that such distributions are quite broad: the size of a depth-$d$ tree can be as large as $2^d = \mathrm{quasipoly}(n)$, corresponding to the mixture of that many subcubes.}  

The proof of~\Cref{thm:lift}   readily extends to lift {\sl agnostic} uniform-distribution  learners~\cite{Hau92,KSS94} to agnostic distribution-free ones.   As mentioned in the introduction, many algorithms for learning under the uniform distribution are  obtained through Fourier-analytic meta-algorithms.  By~\Cref{thm:lift}, we now have analogues of these meta-algorithms for distributions with low decision tree complexity.  

\paragraph{Estimating influences efficiently.} We have stated~\Cref{thm:lift} assuming that influences of variables on distributions can be computed exactly in unit time.  In the body of the paper we show how these quantities can be approximated to sufficiently high accuracy given access to $\mathcal{D}$.  For monotone distributions we show how this can be done using only samples from $\mathcal{D}$.  For general distributions, we use the notion of {\sl subcube conditioning samples}, proposed in~\cite{CRS15,BC18} and subsequently studied in~\cite{CCKLW21,CJLW21}: in this model, the algorithm specifies a subcube of $\bits^n$ and receives a draw $\bx \sim \mathcal{D}$ conditioned on $\bx$ lying in the subcube.  (For general distributions, estimating influences based only on samples is intractable as it can be easily seen to require $\Omega(\sqrt{2^n})$ many samples.) 

\paragraph{A more general result.} We obtain~\Cref{thm:lift} as a corollary of a more general result which shows how every uniform-distribution learner $\mathcal{A}$ that is robust to {\sl distributional noise} can be lifted in a way that the resulting runtime depends only on $\mathcal{A}$'s noise tolerance and not its sample complexity.  \Cref{thm:lift} follows since every $m$-sample algorithm is automatically tolerant to an $O(1/m)$ amount of distributional noise.  We achieve better parameters for algorithms that are tolerant to higher amounts of noise.


\subsection{Second main result: Learning decision tree distributions}

Our algorithm for~\Cref{thm:lift} proceeds in a two-stage manner: we first learn the decision tree structure of $\mathcal{D}$ and then use the uniform-distribution learner to learn $f$ restricted to each of the leaves of the tree.   To carry out the first stage, we give an algorithm that learns the {\sl optimal} decision tree decomposition of a distribution $\mathcal{D}$:

\begin{theorem}[Learning decision tree distributions]
\label{thm:learn-DT} 
    Let $\mcD$ be a distribution that is representable by a depth-$d$ decision tree.
    There is an algorithm that returns a depth-$d$ tree representing a distribution $\mcD'$
    such that $\TV(\mcD,\mcD') \le \eps$, with high probability over the draw of samples, and with running time and sample complexity $\poly(n) \cdot (d/\eps)^{O(d)}$. For monotone distributions, the algorithm only uses random samples from $\mcD$, and for general distributions, it uses subcube conditional samples.
\end{theorem}

{{\Cref{thm:learn-DT} is a result of independent interest in distribution learning.}}  Even setting aside properness, the performance guarantees of our algorithm  {{improves, quite dramatically, the prior state of the art for learning decision tree distributions and circumvents existing hardness results}}.  Aliakbarpour, Blais, and Rubinfeld~\cite{ABR16} gave an $n^{O(k)}$ time algorithm for learning $k$-junta distributions, which implies an $n^{O(2^d)}$ time algorithm for learning depth-$d$ decision tree distributions. Chen and Moitra~\cite{CM19} gave an $s^{s^3}\cdot n^{O(\log s)}$ time algorithm algorithm for learning the mixture of $s$ subcubes, which implies a $2^{2^{O(d)}}\cdot n^d$ time algorithm for learning depth-$d$ decision tree distributions.  Finally, Feldman, O'Donnell, and Servedio~\cite{FOS08} gave an $n^{O(m^3)}$ time algorithm algorithm for learning the mixture of $m$ product distributions, which implies an $n^{2^{O(d)}}$ time algorithm for learning depth-$d$ decision tree distributions.  These runtimes all have a doubly-exponential dependence on $d$, whereas ours only depends exponentially on $d$. { Note that the runtime of any  algorithm must have at least an exponential dependence on $d$ since that is the description length of a depth-$d$ decision tree.}

None of these prior algorithms are proper, and that fact that ours is is crucial to the application to~\Cref{thm:lift}: the decision tree structure specifies a decomposition of $\bits^n$ into disjoint subcubes, and it is on these subcubes that we run our uniform-distribution learner.  

\paragraph{Circumventing hardness results.} {{A novel aspect of~\Cref{thm:learn-DT} is that it sidesteps existing hardness results for learning decision tree distributions.}} \cite{FOS08} showed that the problem of learning depth-$d$ decision tree  distributions is as hard as that of learning depth-$d$ decision tree {\sl functions} under the uniform distribution. Despite significant efforts for over three decades, the current fastest algorithm for the latter problem runs in time $n^{O(d)}$~\cite{EH89} and improving on this is a  longstanding challenge of learning theory. (It contains as a special case the junta problem~\cite{BL97}---learning $k$-juntas in time better than $n^{O(k)}$---itself already a notorious open problem.) ~\Cref{thm:learn-DT} shows that this barrier can be circumvented in two different ways: by giving the algorithm access to subcube conditional samples, and by considering monotone distributions.

\paragraph{Proof Overview of~\Cref{thm:learn-DT}.}
To describe the intuition behind~\Cref{thm:learn-DT} and the role that influence plays in its proof, we begin by considering the following elementary equations: 
\begin{align} 
\Ex_{\bb \sim \bits}\big[\Inf((f_{\mathcal{D}})_{x_i=\bb})\big] &= \Inf(f_\mathcal{D}) - \Inf_i(f_\mcD) \label{eq:influence-drops} \\
2\cdot  \TV(\mathcal{D},\mathcal{U})  & \le \Inf(f_{\mathcal{D}}),  \label{eq:influence upper bounds distance to uniformity} 
\end{align} 
where $(f_\mathcal{D})_{x_i=b}$ denotes the restriction of $f_\mathcal{D}$ by fixing $x_i$ to $b$. 
\Cref{eq:influence-drops}, which follows from the definition of influence, says that the total influence of $f_\mathcal{D}$ drops by $\Inf_i(f)$ when restricted by $x_i$.  \Cref{eq:influence upper bounds distance to uniformity}, which is a consequence of the Efron--Stein inequality, says that the total influence of a distribution upper bounds its distance to uniformity.  

Together, they  suggest a simple and natural algorithm for learning decision tree distributions: build a decision tree hypothesis for $\mathcal{D}$ greedily by iteratively querying the most influential variable.  After sufficiently many stages, the total influence of the conditional distributions at most leaves will be close to zero (by~\Cref{eq:influence-drops}), which in turn means that most leaves will be close to uniform (\Cref{eq:influence upper bounds distance to uniformity}).  Indeed, this intuition can be formalized using the techniques in this work to get an algorithm that learns depth-$d$ decision tree distributions with depth-$O(d^2)$ decision tree hypotheses in time~$2^{O(d^2)}$.  

To obtain the improved parameters of~\Cref{thm:learn-DT}, we consier a generalization of this different algorithm: instead of splitting on the single most influential variable, we consider all $O(d)$ most influential ones as candidate splits.  While this involves searching over more candidates at each split, we will show that at least one of the choices leads to a high accuracy hypothesis at depth $d$ instead of $O(d^2)$, resulting in a smaller search space of $d^{O(d)}$ instead of $2^{O(d^2)}$.  

Our approach is inspired by a recent algorithm of Blanc, Lange, Qiao, and Tan for properly learning decision tree {\sl functions} under the uniform distribution~\cite{BLQT21focs}.  Our analysis builds on and extends theirs to the setting of unsupervised learning, which poses a number of challenges that we have to overcome.  First, while highly accurate estimates of influences can be easily obtained with membership queries to the function (in our case, the pmf of $\mathcal{D}$), subcube conditioning samples  provide more limited and coarse-grained information about $\mathcal{D}$.  Our algorithms for estimating influences with subcube conditioning samples, and from samples alone for monotone distributions, could see further utility in other problems.  Second, while it is easy to estimate how close a {\sl function} is to a constant (relatedly, how close two functions are) via random sampling, the analogous problem of estimating the distance of {\sl distribution} to uniformity is an intractable problem: for a distribution over $\bits^n$, the sample complexity of estimating its distance to uniformity is $\Theta(2^n/n)$~\cite{VV11}.  There are no known improvements using subcube conditioning samples (though~\cite{BC18,CCKLW21} give efficient uniformity {\sl testers}), and the best known algorithm for monotone distributions uses $2^{n-\Theta(\sqrt{n}\log n)}$ samples~\cite{RV20}.  We sidestep this barrier by showing how total influence---for which we provide efficient estimators---can be used as a good proxy for distance to uniformity. 


\subsection{Other related work} 

\paragraph{Conditional samples in distribution learning and testing.}  The subcube conditioning model falls within a recent line of work on the power of {\sl conditional samples} in distribution learning and testing.  In this more general model, which was independently introduced by Chakraborty, Fischer, Goldhirsh, and Matsliah~\cite{CFGM16} and Canonne, Ron, and Servedio~\cite{CRS15}, the algorithm can specify an arbitrary subset of the domain and receive a sample conditioned on falling within this subset.  Since its introduction, a large number of works have designed conditional sample algorithms, in distribution learning and testing~\cite{Can15,FJOPS15,ACK15chasm,SSJ17,BCG19, FLV19,CJLW21,CCKLW21} and beyond~\cite{ACK15,GTZ17,GTZ18}.  Our results add to this line of work, and further reinforce the message that conditional samples (indeed, even just subcube conditional samples) can be used circumvent sample complexity lower bounds in a variety of settings. 

Other access models to distributions include queries to the pmf or cdf (the {\sl evaluation oracle} model)~\cite{BDKR05,GMV06,RS09,CR14} and giving the algorithm probability revealing samples~\cite{OS18}.

\paragraph{Semi-supervised learning.}
There is extensive research in the statistical machine learning literature on leveraging unlabeled examples to improve learning.  Much of this work focuses on improving sample complexity or convergence rates of existing algorithms using additional unlabelled data \cite{pmlr-v99-gopfert19a,BdLP08}. In contrast, our work is aimed at creating \emph{computational and sample efficient} algorithms from existing ones that are only guaranteed to work under ``nice" distributions. Results of this flavour have  been studied in limited ways, e.g.~transforming a 1d-uniform learning algorithm to one that works on any 1d continuous distribution \cite{BdLP08}.

\paragraph{The work of~\cite{BOW10}.} In~\cite{BOW10}, Blais, O'Donnell, and Wimmer gave an algorithm for performing polynomial regression under arbitrary product distributions over $\bits^n$.  As an application, they showed how their algorithm can be lifted to {\sl mixtures} of product distributions via the algorithm of~\cite{FOS08} for learning mixtures of product distributions.  

Like our work, this is also an example where algorithms for learning with respect to a ``simple" distribution (the uniform distribution in our case and product distributions in~\cite{BOW10}'s case) can be lifted to more complex ones (decision tree distributions in our case and mixtures of product distributions in~\cite{BOW10}'s case), via a distribution learning algorithm that decomposes the more complex one into simple ones (\Cref{thm:learn-DT} in our case and~\cite{FOS08}'s algorithm in~\cite{BOW10}'s case).  The quantitative details of our transformations are incomparable: our algorithm for learning depth-$d$ decision tree distributions run in $\poly(n)\cdot d^{O(d)}$ time, whereas~\cite{FOS08}'s algorithm for learning the mixture of $m$ product distributions run in $n^{O(m^3)}$ time.  {{(Recall that a depth-$d$ decision tree induces a mixture of as many as $m=2^d$ product distributions.)}}

\section{Discussion and future work}

We view our work as part of  two broader and potentially fruitful approaches to PAC learning.  Much of the progress in the field thus far has been guided by the design of efficient learning algorithms for successively more expressive concept classes, as measured according to various notions of function complexity; this was the approach advocated in Valiant's pioneering paper and other early works. For example, on one such axis we have small-width conjunctions as a special case of small juntas, which are in turn a special case of small-depth decision trees, which are in turn a special case of small-width DNFs, and so on, and the field seeks to design efficient algorithms for each of these classes.  We believe that it is equally natural to make progress along a separate dimension, with respect to various notions of distribution complexity. The overall goal can then be cast as that of learning successively more expressive concept classes with respect to successively more expressive distributions.

Next, our algorithm is just one instantiation of a general two-stage approach to learning that is studied in the semi-supervised literature (see e.g.~\cite{BFB10}): the first gathers information about the underlying distribution, and the second exploits this distributional information to learn the target function.  It would be interesting to develop more computationally efficient examples of such a two-stage approach. More broadly, there should be much to be gained from using the insights of distribution learning to counter the difficulty of distribution-free PAC learning, the crux of which is the potential nastiness of the unknown distribution. 

Finally, looking beyond PAC learning, a similar gulf exists between uniform-distribution and distribution-free {\sl testing} of function properties. The original model of property testing was defined with respect to the uniform distribution~\cite{RS96,GGR98} and much of the ensuing research has focused on this setting, with the distribution-free variant receiving increasing attention in recent years. Can uniform-distribution testers be lifted generically to the distribution-free setting?   A concrete avenue towards such a result would be via a variant of our distribution decomposition lemma that runs in sublinear time.

\section{Preliminaries}

\paragraph{Notation.} Given an input $x \in \bits^n$, coordinate $i \in [n]$, and setting $b \in \bits$, we use $x_{i = b}$ to refer to the input $x$ with the $i^{\text{th}}$ coordinate overwritten to take the value $b$. 
Similarly, given a sequence of (coordinate, value) pairs 
$\pi = \{(i_1, b_1), \ldots, (i_k, b_k) \}$, 
we use $x_{\pi}$ to represent $x$ with the coordinates in $\pi$ overwritten/inserted with their respective values. 

Given a function $f: \bits^n \to \R$,  we denote the \emph{restriction} $f_{i = b}: \bits^n \to \R$ to be the function that maps $x$ to $f(x_{i = b})$. We define the restriction $f_\pi$ analogously. 


\begin{definition}[Decision trees (DT)]
 A decision trees $T : \bits^n \to \R$, is a binary tree whose internal nodes query a particular coordinate, and whose leaves are labelled by values. Each instance $x \in \bits^n$ follows a unique root-to-leaf path in $T$: at any internal node, it follows either the left or right branch depending on the value of the queried coordinate, until a leaf is reached and its value is returned. 
 
\end{definition}

 The set of leaves $\ell \in \mathrm{leaves}(T)$ therefore form a partition of $\bits^n$,  with each leaf having $2^{n - |\ell|}$ elements, where $|\ell|$ is the depth of the leaf.  Every leaf $\ell$ also corresponds to a sequence of  (coordinates, value) pairs  $\pi(\ell)$ that lead to the leaf. For a function $f$, will sometimes use the shorthand $f_\ell$ to mean the restriction $f_{\pi(\ell)}$.

\begin{definition}[Decision tree distribution]
    We say that a distribution $\mathcal{D} : \bits^n \to [0,1]$ is representable by a depth-$d$ DT, if its pmf is computable by a depth-$d$ decision tree $T$. 
    Specifically, each leaf $\ell$ is labelled by a value $p_\ell$, so that $\mcD(x) = p_\ell$ for all $x \in \ell$. This means that the conditional distribution of all points that reach a leaf is uniform. Moreover, since $\mcD$ is a distribution, we have: $\sum_{\ell \in \mathrm{leaves}(T)} 2^{n - |\ell|} \cdot p_\ell  = 1$.
   
\end{definition}
   
   For a given leaf $\ell$, we will write  $\mcD_\ell : \bits^{n - |\ell|} \to [0,1]$ to represent the conditional distribution of $\mcD$ at the leaf $\ell$, so that for any $x \in \bits^{n - |\ell|}$, we have $\mcD_\ell(x) = \mcD(x_{\pi(\ell)}) /\Pr_{y \sim \mcD}[y \in \ell]$.

    We will often scale up the pmfs of our distributions by the domain size, since it makes our analysis easier. As such, we also define the weighting function:

\begin{definition}[Weighting function of distribution]
Let $\ArbDist$ be an arbitrary distribution over $\bits^m$. We define the weighting function: 
\[
f_\ArbDist(x) \coloneqq 2^m \cdot \ArbDist(x).
\]

\end{definition}

\begin{definition}[Monotone distribution]
We furthermore say that a distribution $\mcD$ is monotone if its pmf is monotone: for $x, y \in \bits^n$, if $x_i \leq y_i$ for all $i \in [n]$, then $\mcD(x) \leq \mcD(y)$.
\end{definition}

\begin{definition}[TV Distance]
    For two distributions $\mcP, \mcQ$ over a countable domain $\mcX$, we define the total variation distance:
    \begin{equation*}
        \TV(\mcP, \mcQ) = \frac{1}{2}\sum_{x \in \mcX} | \mcP(x) - \mcQ(x)|= \frac{1}{2}\| \mcP - \mcQ\|_1 .
    \end{equation*}
\end{definition}

\begin{definition}[$\ell_1$ Influence]
For any function $f : \bits^n \to \R$, the influence of the $i$-th variable on $f$ is given by:
\begin{equation*}
\Inf_{i}(f) \coloneqq \Ex_{\bx \sim \mcU^n} \big[| f(\bx) - f(\bx^{\sim i}) |\big],
\end{equation*}   
where $\bx^{\sim i}$ denotes $\bx$ with the $i$-th coordinate re-randomised. Note that the influence of a function is defined with respect to the uniform distribution over its domain.

We further define the total influence as the sum of influences over all variables:
$
\Inf(f) \coloneqq \sum_{i=1}^n \Inf_i(f).
$

\end{definition}

\begin{fact}[Influence $\equiv$ correlation for monotone functions]\label{fact:inf_eq_corr}
    Let $f : \bits^n \to \R$ be a monotone function. Then 
    \begin{equation*}
        \Inf_i(f) = \Ex_{\bx \sim \mcU^n}[f(\bx) \cdot \bx_i] .
    \end{equation*}
\end{fact}

\begin{definition}[$\ell_1$ Variance]
\label{def:variance} 
    For any function $f : \bits^n \to \R$,
    
    \begin{equation*}
        \Var^{(1)}(f) \coloneqq \Ex_{\bx, \by \sim \mcU^n} |f(\bx) - f(\by)|.
    \end{equation*}
    
    We will also sometimes use a different definition of variance, given by the mean absolute deviation of $f$:
    \begin{equation*}
    \Var_{\mu}(f) \coloneqq \Ex_{\bx \sim \mcU^n} |f(\bx) - \Ex[f]|.
    \end{equation*}
\end{definition}

These two definitions are equivalent, up to constant factors:

\begin{lemma}\label{lem:var_defs}
    For a function $f : \bits^n \to \R$, 
    \begin{align*}
           \Var_{\mu}(f) \leq \Var^{(1)}(f) \leq 2\Var_{\mu}(f) 
    \end{align*}    
\end{lemma}

\begin{proof}
    The second part follows immediately from the triangle inequality and the first is an application of Jensen's:
    \[ \Var_{\mu}(f) =  \Ex_{\bx \sim \mcU}\left[\Big| \Ex_{\by \sim \mcU} [f(\bx) -  f(\by)]\Big|\right]\\
        \leq  \Ex_{\bx, \by \sim \mcU}\left|  f(\bx) -  f(\by)\right|.  \qedhere \]  
\end{proof}

\begin{definition}[Sensitivity]
For a function $f:\bits^n \to \R$ and $x \in \bits^n$, the sensitivity of $f$ at $x$ is defined to be
\[s(f,x) = \sum_{i=1}^n \Ind[f(x) \ne f(x^{\oplus i})].\] 
Furthermore, the sensitivity of $f$ is given by its maximum sensitivity over all points:
\[s(f) = \max_{x \in \bits^n} \{ s(f,x) \}.\]
\end{definition}

Note that the sensitivity of a decision tree is at most the depth of the decision tree, since any point can only be sensitive to the coordinates queried on its root-to-leaf path.

\subsection{Useful inequalities}
We present some useful inequalities for boolean functions.

\begin{lemma}[Efron-Stein]\label{lem:efron_stein}
For any function $f: \bits^n \to \R$:
\begin{align*}
    \Var^{(1)}(f) \leq \Inf(f).
\end{align*}
\end{lemma}

\begin{lemma}[Total influence and sensitivity]\label{lem:inf_s_var}
    For any function $f : \bits^n \to \R$:
    \begin{align*}
        \Inf(f) \leq 2s(f) \cdot \Var^{(1)}(f).
    \end{align*}
\end{lemma}

\begin{proof}
    Let $s = s(f)$ be the sensitivity of $f$ and consider the set, $\mathrm{snbr}(x) = \{i \in [n] \mid f(x) \ne f(x^{\oplus i})\}$. By assumption, $|\mathrm{snbr}(x)| \leq s$. We define a coupling $(\bx, \by) \sim \pi$ s.t. $\by$ is often in $\mathrm{snbr}(\bx)$,  but the marginal distributions $\pi(\bx)$ and $\pi(\by)$ are still uniform. First, sample $\bx \sim \mcU$. Then, sample $\by$ given $\bx$ as follows: for each $i \in \mathrm{snbr}(\bx)$, let $\by = \bx^{\oplus i}$ (i.e. flip the $i$-th coordinate of $\bx$) with probability $1/s$, and with the remaining $ 1 - |\mathrm{snbrs}(\bx)|/s$ probability, take $\by = \bx$. It is easy to see that the marginal distribution over $\by$ is still uniform.
    
    Unrolling the definition of influence, we have:
    \begin{align*}
        \Inf(f) 
            &= \sum_{i=1}^n \Ex_{\bx \sim \mcU} |f(\bx) - f(\bx^{\oplus i})|  \\
            &= \Ex_{\bx \sim \mcU} \left[ \sum_{i=1}^n |f(\bx) - f(\bx^{\oplus i})| \right ]  \\
            &= \Ex_{\bx \sim \mcU} \left[ \sum_{i \in \mathrm{snbr}(\bx)} |f(\bx) - f(\bx^{\oplus i})| \right ] \tag{only consider nonzero terms}\\
            &=  \Ex_{\bx \sim \mcU} \left[ s \cdot \sum_{i \in \mathrm{snbr}(\bx)} \frac{|f(\bx) - f(\bx^{\oplus i})|}{s}  \right ] \\
            &=  \Ex_{\bx \sim \pi} \left[ s \cdot \Ex_{\by \sim \pi(\cdot | \bx)} |f(\bx) - f(\by)|  \right ] \tag{definition of coupling $\pi$}\\
            &=  s \cdot\Ex_{(\bx, \by) \sim \pi} |f(\bx) - f(\by)|  \\
            &\leq  s \cdot \Ex_{(\bx, \by) \sim \pi}  |f(\bx) - \E[f]|+ s \cdot \Ex_{(\bx, \by) \sim \pi} |f(\by) - \E[f]|  \tag{triangle inequality}\\
            &=  2s \cdot \Var_\mu(f) \tag{marginal distributions of $\pi$ are uniform }\\
            &\leq  2s \cdot \Var^{(1)}(f) \tag{\Cref{lem:var_defs}}.
    \end{align*}
\end{proof}

\section{Our algorithmic decomposition lemma}
\label{sec:decomp}

Here we present an algorithm that constructs a decision tree of depth $d$ for a a distribution $\mcD$, and analyze its correctness and complexity. 
Throughout this section, we assume access to an oracle that gives the exact influences of variables in $f_\mcD$ or any of its restrictions. 
In the next sections we will show that the influences can be estimated from random examples for monotone distributions, and from subcube conditional examples for general distributions. 

\begin{theorem}[Learning decision tree distributions]
    \label{thm:decompose}
    Let $\mcD$ be a distribution that is representable by a depth-$d$ decision tree.
    The algorithm $\BuildDT$ returns a depth-$d$ tree representing a distribution $\mcD'$
    such that $\TV(\mcD,\mcD') \le \eps$ w.h.p..
    Given access to a unit time influence oracle, its running time is $n \cdot (d/\eps)^{O(d)}$.
\end{theorem}

The algorithm $\BuildDT$ is an exhaustive search over a subset of depth-$d$ decision trees. 
We characterize this subset as follows: 

\begin{definition}[Everywhere $\tau$-influential]
Let $T$ be a tree and $\nu$ be an internal node with root variable $i(\nu)$. $T$ is \emph{everywhere $\tau$-influential} with respect to some $f$ if for every $\nu \in T$, we have $\Inf_{i(\nu)}(f_\nu) \ge \tau$.
\end{definition}

\begin{figure*}[t] 
  \captionsetup{width=.9\linewidth}

\begin{tcolorbox}[colback = white,arc=1mm, boxrule=0.25mm]
\vspace{3pt} 

$\BuildDT(\mcD, \pi, d, \tau)$:

\begin{itemize}
    \item[]\textbf{Input:} Random examples from $\mcD$, restriction $\pi$, influence oracle for $(f_\mcD)_\pi$, depth parameter $d$, influence parameter $\tau$.
    \item[]\textbf{Output:} A decision tree $T$ that minimizes $\Ex_{\bell \in T}[\Inf((f_\mcD)_{\bell})]$ among all depth-$d$, everywhere $\tau$-influential trees.
\end{itemize}
\begin{enumerate}
    \item Let $S \subseteq [n]$ be the set of variables $i$ such that $\Inf_i((f_\mcD)_\pi) \ge \tau$.
    \item If $S$ is empty or $d=0$, return the leaf labeled with $2^{|\pi|} \cdot \Pr_{\bx \sim \mcD}[\bx\text{ is consistent with }\pi]$.
    \item Otherwise: 
    \begin{enumerate}
        \item For each $i \in S$, let $T_i$ be the tree such that 
            \begin{align*}
                \mathrm{root}(T_i) &= x_i \\
                \textnormal{left-subtree}(T_i) &= \BuildDT(\mcD, \pi \cup \{x_i = -1\}, d-1, \tau) \\
                \textnormal{right-subtree}(T_i) &= \BuildDT(\mcD, \pi \cup \{ x_i = 1\}, d-1, \tau)
            \end{align*}
        \item Return the tree among the $T_i$'s defined above that minimizes $\Ex_{\bell \in T_i}[\Inf((f_\mcD)_{\bell})]$.
    \end{enumerate}
\end{enumerate}

\end{tcolorbox}

\caption{$\BuildDT$ recursively searches for the depth-$d$, everywhere $\tau$-influential tree of minimal influence at the leaves.}
\label{fig:BuildDT}
\end{figure*} 

\subsection{Correctness} 
Here we show that under the oracle assumptions described above, $\BuildDT$ returns a tree within TV distance $\eps$. The proof will rely on the following fact, which relates TV distance to the uniform $\ell_1$ error of the tree with respect to $f_\mcD$. 

\begin{fact}[TV distance = label error]
\label{fact:distance error}
     \begin{align*}
			\TV(\mcD, \mcD') &= \frac{1}{2} \cdot \|\mcD - \mcD'\|_1 \\
							 &= 2^{-(n+1)} \cdot \|2^n \mcD - 2^n \mcD'\|_1 \\
							 &= 2^{-(n+1)} \cdot \|f_\mcD - T'\|_1.
        \end{align*}
\end{fact}

First, we will show that $\BuildDT$ outputs a decision tree $T'$ with small average influence at the leaves.
Then, we will show that this implies that the uniform $\ell_1$ error of $T'$ with respect to $f_\mcD$ is small.
Correctness follows from the equivalence between $2^{-(n+1)}\| f_\mcD- T' \|_1$ and $\TV(\mcD, \mcD')$. \\

The claim that $\BuildDT$ outputs a decision tree $T'$ with small average influence at the leaves extends a lemma from \cite{BLQT21focs}, instantiated here for the metric space $\R$ equipped with the  $\ell_1$-norm:



\begin{lemma}[Theorem 5 of \cite{BLQT21focs}]
\label{lem:blqt} Let $f : \bits^n \to \R$ be representable by a depth-$d$ DT $T$. Then there exists $T^\star$ such that the following are satisfied:

\begin{enumerate}
    \item The size and depth of $T^\star$ are at most the size and depth of $T$,
    \item $T^\star$ is everywhere $\tau$-influential with respect to $f$,
    \item $2^{-n} \cdot \| f - T^\star \|_1 \le d\tau$.
\end{enumerate}
\end{lemma}


In our \BuildDT, we cannot compute $\|f_\mcD - T'\|_1$ and hence cannot search for trees that minimise this error. Instead, we find trees that minimise the expected total influence at the leaves. The following lemma relates these two values:

\begin{lemma}[Expected total influence and $\ell_1$ error]\label{lem:inf_leaves_l1}
Let $f : \bits^n \to \R$ be representable by a depth-$d$ DT, and let $T'$ be any other DT. Then:

\begin{equation*}
    \Ex_{\bell \in T'} [\Inf(f_{\bell})] \leq 4d \cdot 2^{-n} \|f - T'\|_1
\end{equation*}
\end{lemma}

\begin{proof}
	Since $f$ is representable by a depth-$d$ decision tree, its maximum sensitivity (and the sensitivity of each of its leaf restrictions) must be at most $d$. Therefore, for any leaf $\ell \in T'$, \Cref{lem:inf_s_var} asserts that $\Inf(f_{\ell}) \leq  2d \cdot \Var^{(1)}(f_{\ell})$. Moreover, 
	\begin{align*}
	    \Var^{(1)}(f_{\ell}) 
	        &= \Ex_{\bx, \by \sim \mcU^n} |f_\ell(\bx) - f_\ell(\by)| \\
	        &= \Ex_{\bx, \by \sim \mcU^n} |f_\ell(\bx) - T'_\ell + T'_\ell - f_\ell(\by)| \\
	        &\leq 2\cdot \Ex_{\bx \sim \mcU^n} |f_\ell(\bx) - T'_\ell| \tag{Triangle ineq.}\\
	        &= 2\cdot \Ex_{\bx \sim \mcU^n} [|f(\bx) - T'(\bx)| \mid \bx \in \ell].
	\end{align*}
	
	Therefore, 
	\begin{align*}
	    \Ex_{\bell \in T'} [\Inf(f_{\bell})] 
	        &\leq 2d \cdot \Ex_{\bell \in T'} [\Var^{(1)}(f_{\bell})] \\
	        &\leq 4d \cdot \Ex_{\bell \in T'} \Ex_{\bx \sim \mcU^n} [|f(\bx) - T'(\bx)| \mid x \in \ell] \\
	        &= 4d \cdot \Ex_{\bx \sim \mcU^n} |f(\bx) - T'(\bx)| \\
	        &= 4d  \cdot 2^{-n} \cdot \|f(\bx) - T'(\bx)\|_1. \qedhere 
	\end{align*}	
\end{proof}

As a corollary of \Cref{lem:blqt} and \Cref{lem:inf_leaves_l1}, we get our pruning lemma stated in terms of influences:

\begin{corollary}[Pruning lemma with expected total influence at leaves]
\label{cor:influence}
Let $f : \bits^n \to \R$ be representable by a depth-$d$ DT $T$. Then there exists $T^\star$ such that the following are satisfied:

\begin{enumerate}
    \item The size and depth of $T^\star$ are at most the size and depth of $T$,
    \item $T^\star$ is everywhere $\tau$-influential with respect to $f$,
    \item $\Ex_{\bell \in T^\star}[\Inf(f_{\bell})] \leq 4d^2 \tau$.
\end{enumerate}

\end{corollary}

We now move to show that the tree output by \BuildDT \  satisfies $2^{-n} \cdot \|f_\mcD - T'\|_1 \le \eps$.

\begin{claim}
\label{claim:low error}
Let $\mcD$ be a distribution that is representable by a depth-$d$ decision tree, and let $T$ be the output of $\BuildDT$, with $\tau = \eps/8d^2$. Then, with high probability, $2^{-n} \cdot \|f_\mcD - T \|_1 \le \eps$.
\end{claim}

\begin{proof}
    First, we claim that $T$ minimizes $\E_{\bell \in T}[\Inf((f_\mcD)_{\bell})]$ among all depth-$d$, everywhere $\tau$-influential trees.
    This claim holds by induction on $d$: since 
    \[\E_{\bell \in T}[\Inf((f_\mcD)_{\bell})] = \lfrac{1}{2}(\E_{\bell \in T_{\mathrm{left}}}[\Inf((f_\mcD)_{\bell})] + \E_{\bell \in T_{\mathrm{right}}}[\Inf((f_\mcD)_{\bell})]),\] 
    each candidate $T_i$ minimizes influence among all depth-$d$, everywhere $\tau$-influential trees with $x_i$ at the root 
    under the inductive assumption that $T_{\mathrm{left}}$ and $T_{\mathrm{right}}$ minimize influence for depth-$(d-1)$ trees. 
    Then $\BuildDT$ chooses the tree of smallest influence among all the candidate $T_i$'s,
    so it minimizes total influence at leaves among all trees in its search space of depth-$d$, $\tau$-influential trees. 
    
    Since \Cref{cor:influence} establishes the existence of a tree with average total influence at leaves $\le \eps/2$, 
    it follows that the influence $T$'s influence is also $\le \eps/2$.
    Furthermore, we may assume by standard Hoeffding bounds that with a sample size of $\poly(2^d, 1/\eps)$, each leaf's value estimate of $\E[f_\mcD(\bx)~|~\bx\text{ is consistent with }\ell] = 2^{|\ell|} \cdot \Pr_{\bx \sim \mcD}[\bx\text{ is consistent with }\ell]$ is accurate to within $\pm \eps/2$ w.h.p..

    We can now show that $2^{-n} \cdot \|f_\mcD - T\|_1 \le \eps$. 
    Throughout this section, $x \in \ell$ will stand in as shorthand for ``$x\text{ consistent with the restriction at }\ell$''.
    \begin{align*}
    2^{-n} \cdot \|f_\mcD - T\|_1 &= 2^{-n} \cdot \sum_{x \in \bits^n} |f_\mcD(x) - T(x)|\\
					 &= 2^{-n} \cdot \sum_{\ell \in T}\sum_{x \in \ell} |(f_\mcD)_\ell(x) - T_\ell(x)| \\
					 &\le 2^{-n} \cdot \sum_{\ell \in T}\sum_{x \in \ell} |(f_\mcD)\ell(x) - \E[(f_\mcD)_\ell]| + |\E[(f_\mcD)_\ell] - T_\ell(\bx)|  \\
					 &= 2^{-n} \cdot \big(\sum_{\ell \in T} 2^{n-|\ell|} \cdot \Var_\mu((f_\mcD)_\ell) + \sum_{\ell \in T}2^{n-|\ell|} \cdot |\E[(f_\mcD)_\ell] - T_\ell|\big) \\
				     &\le \E_{\ell \in T} \Inf((f_\mcD)_\ell)) + \E_{\ell \in T} \big [|\E[(f_\mcD)_\ell(\bx)] - T_\ell| \big ] \tag{\Cref{lem:efron_stein}}\\
				     &\le \frac{\eps}{2} + \frac{\eps}{2} = \eps.\qedhere 
    \end{align*}
\end{proof}

We can now prove correctness and time bounds for $\BuildDT$.
\begin{proof}[Proof of \Cref{thm:decompose}]
Correctness follows by combining \Cref{claim:low error} and \Cref{fact:distance error}. To bound the running time, first we note that for all restrictions $\pi$, there are at most $d/\tau = 8d^3/\eps$ variables of influence at least $\tau$. This is because all restrictions of $f_\mcD$ are depth-$d$ decision trees, and for any depth-$d$ decision tree, the sum of all variable influences is at most $d$. 

Then the number of recursive calls to $\BuildDT$ is at most $(8d^3/\eps)^d$. Each call at an internal node makes $n$ calls to the unit-time influence oracle, and each call at a leaf processes $\poly(2^d, 1/\eps)$ samples to estimate the leaf label. Thus, the total running time is $n\cdot (d/\eps)^{O(d)}$, as desired.
\end{proof}

\section{Algorithms for computing distributional influences}


In \Cref{sec:decomp}, we assumed the ability to exactly compute influences of $f$ and its restrictions in unit time. In this section, we show how to instead estimate the influences from samples, when the distribution is monotone or when we have access to subcube conditional samples. Just as in \cite{BLQT21focs}, our proof only requires estimates to be accurate to $\pm \min(\tau/4, \eps/n)$. Letting $\InfEst_i(f)$ denote such an estimate of $\Inf_i(f)$, the pseudocode in \Cref{fig:BuildDT} and proof of \Cref{thm:decompose} is modified as follows.
\begin{enumerate}
    \item We modify $S$ to include variables $i$ such that $\InfEst_i((f_{\mcD})_\pi) \geq 3\tau/4$. Since the estimate is accurate to $\pm \tau/4$, this is guaranteed to include all variables with influence $\geq \tau$, and furthermore will only include variables with influence at least $\tau/2$. Therefore, the total size of $S$ is at most $\frac{d}{\tau/2}$, which is only a constant factor (of $2$) larger than in the proof of \Cref{thm:decompose} which assumed perfect influence oracles, and there does not affect asymptotic runtime or sample complexity.
    \item It returns the tree $T_i$ that minimizes $\Ex_{\bell \in T_i}[\sum_{i \in [n]}\InfEst_i((f_\mcD)_{\bell})]$. Since each estimate is accurate to $\eps/n$, this estimate of total influence of $(f_\mcD)_{\bell})$ will be accurate to $\pm \eps$. Then, in \Cref{claim:low error}, rather than $\BuildDT$ building a tree $T$ minimizes $\E_{\bell \in T}[\Inf((f_\mcD)_{\bell})]$ among all depth-$d$, everywhere $\tau$-influential trees, $T$ (roughly) minimizes $\E_{\bell \in T}[\sum_{i \in [n]}\InfEst_i((f_\mcD)_{\bell})]$ among all depth-$d$, everywhere $\tau$-influential trees. More formally, $T$ will either be the best depth-$d$, everywhere $\tau$-influential trees, or better, as we know its searches over all variables with $\InfEst_i((f_{\mcD})_\pi) \geq 3\tau/4$ which is guaranteed to include variables with influence $\geq \tau$, but can include more variables. Finally, since the estimates of total influence are accurate to $\pm \eps$, using $\InfEst$ rather than $\Inf$ can only incur at most $2\eps$ additive error, which is a constant factor in the analysis.
\end{enumerate}

Given any distribution $\mcD$ over $\bits^n$, we need to estimate $\Inf_i((f_\mcD)_{\ell})$ for any restriction $\ell$ of $f_\mcD$ and $i \in [n] - \ell$. In this section, we will instead show how to compute $\Inf_i(f_{\ArbDist})$ for any distribution $\ArbDist$ over $\bits^m$.  We can then use our estimators with $\ArbDist = \mcD_\ell$  to obtain the necessary answers using the following fact:

\begin{fact}
    \label{fact:inf-estimates-scaling}
    For any distribution $\mcD$ over $\bits^n$,  restriction $\ell$ of $\bits^n$, and $i \in [n] - \ell$:
    \begin{equation*}
        \Inf_i((f_\mcD)_{\ell}) = 2^{|\ell|} \cdot \Prx_{x \sim \mcD}[x \in \ell] \cdot \Inf_i(f_{\mcD_\ell}).
    \end{equation*}
\end{fact}
\begin{proof}
    Let $w_\ell = \Prx_{\bx \sim \mcD}[\bx \in \ell]$. We have: 
    \begin{align*}
        \Inf_i((f_\mcD)_{\ell}) 
            &= \Ex_{\bx \sim \mcU^n} \left[\left| f_\mcD(\bx) - f_\mcD(\bx^{\sim i}) \right | \mid x \in \ell \right] \\
            &= 2^n \cdot \Ex_{\bx \sim \mcU^n} \left[\left| \mcD(\bx) - \mcD(\bx^{\sim i}) \right | \mid x \in \ell \right] \\
            &= 2^n \cdot w_\ell \cdot \Ex_{\bx \sim \mcU^n} \left[\left| \frac{\mcD(\bx)}{w_\ell} - \frac{\mcD(\bx^{\sim i})}{{w_\ell}} \right | \mid x \in \ell \right]  \\
            &= 2^n \cdot w_\ell \cdot \Ex_{\by \sim \mcU^{n - |\ell|}} \left| \mcD_\ell(\by) - \mcD_\ell(\by^{\sim i}) \right |   \\
            &= 2^{|\ell|} \cdot w_\ell \cdot \Ex_{\by \sim \mcU^{n - |\ell|}} \left| f_{\mcD_\ell}(\by) - f_{\mcD_\ell}(\by^{\sim i}) \right |   \\
            &= 2^{|\ell|} \cdot w_\ell \cdot \Inf_i ( f_{\mcD_\ell}). \qedhere
    \end{align*}
\end{proof}

    Because of \Cref{fact:inf-estimates-scaling}, for any restriction $\ell$ of depth at most $d$, to estimate $ Inf_i((f_\mcD)_{\ell})$ to accuracy $\pm \eps$, it is sufficient to estimate $\Inf_i(f_{\mcD_\ell})$ to accurate $\pm \eps/2^d$. This $2^d$ factor is dominated by the $d^{O(d)}$ term in \Cref{thm:learn-DT}, so we are free to do the later.

\subsection{Monotone distributions using samples} 

Let $\ArbDist$ be an arbitrary distribution over $\bits^m$.  If $\ArbDist$ is monotone, the influences of $f_\ArbDist$ can be efficiently computed directly from samples of $\ArbDist$, via an estimate of bias:



\begin{lemma}[Estimating influence using bias]\label{lem:bias_inf}

If $\ArbDist$ is monotone,

\begin{equation*}
        \Inf_i(f_\ArbDist) = \Ex_{\bx \sim \ArbDist}[\bx_i].
    \end{equation*}
\end{lemma}
\begin{proof}
    Using \Cref{fact:inf_eq_corr}, 
    \begin{align*}
        \Inf_i(f_\ArbDist) 
        &=   \Ex_{\bx \sim \mcU^m}[f_\ArbDist(\bx) \cdot \bx_i]  \\ 
        &= \sum_{x \in \bits^m} 2^{-m} \ f_\ArbDist(x) \cdot x_i   \\
        &= \sum_{x \in \bits^m} \ArbDist(x) \cdot  x_i \\
        &=  \Ex_{\bx \sim \ArbDist}[ \bx_i ] . \qedhere 
    \end{align*} 
\end{proof}

As a simple application of the above and Hoeffding's inequality, we obtain the following corollary.
\begin{corollary}
[Estimating influences of monotone distributions]
\label{cor:high-accuracy-monotone}
    For any $\eps, \delta > 0$, there is an efficient algorithm that given unknown monotone distribution $\ArbDist$, computes an estimate of $\Inf_i(f_{\ArbDist})$ to accuracy $\pm \eps$ with probability at least $1 - \delta$ using $O(\log(1/\delta)/ \eps^2)$ random samples from $\ArbDist$.
\end{corollary}
Recall that for \Cref{thm:decompose}, we only need the influence estimates to be accurate to $\pm \poly(2^{-d}, \tau, \eps,1/n)$. Setting $\tau = O(\eps/d^2)$ and union bounding over $n \cdot (d/\eps)^{O(d)}$ calls to the influence oracle gives a sample complexity of $\poly(n, 2^d, \eps, \log(1/\delta))$. The running time is still dominated by the number of recursive calls.


\begin{corollary}[Sample complexity of $\BuildDT$ for monotone distributions]
The algorithm $\BuildDT(\mcD, \varnothing, d, \lfrac{\eps}{2d^2})$, given $\poly(n,2^d, 1/\eps, \log(1/\delta))$ random examples from a monotone distribution $\mcD$, runs in $\poly(n) \cdot (d/\eps)^{O(d)} \cdot \log(1/\delta)$ time and outputs a distribution within TV distance $\eps$ of $\mcD$. The algorithm fails with probability at most $\delta$.
\end{corollary}

\subsection{Beyond monotone distributions using subcube conditional sampling}
We also design an influence estimator for arbitrary distributions $\ArbDist$ over $\bits^m$ using subcube conditional sampling.

\begin{figure}[H]
  \captionsetup{width=.9\linewidth}
\begin{tcolorbox}[colback = white,arc=1mm, boxrule=0.25mm]
\vspace{3pt} 

$\InfEst(\ArbDist, i, \eps)$:  \vspace{6pt} \\
\textbf{Input:} A distribution $\ArbDist$ over $\bits^m$, coordinate $i \in [m]$, and bias parameter $\eps$. \\
\textbf{Output:} An estimate of $\Inf_i(f_{\ArbDist})$ that has bias at most $\eps$.

\begin{enumerate}
    \item Sample a random $\bx \sim \ArbDist$ and define the subcube
    \begin{equation*}
        S \coloneqq \{\bx\} \cup \{\bx^{\oplus i}\}
    \end{equation*}
    where $x^{\oplus i}$ is $x$ with the $i^{\text{th}}$ bit flipped.
    \item \label{step:est-p}Take $\left\lceil1/\eps^2\right\rceil$ independent samples from $\ArbDist$ conditioned on the output being in $S$, and let $p$ be the fraction of those samples that equal $\bx$.
    \item Output $ |p - (1 - p)|$.
\end{enumerate}
\end{tcolorbox}
\caption{Pseudocode for estimating the influence of a variable on a distribution's weighting function.}
\label{fig:inf-est}
\end{figure}

\begin{proposition}[{\sc InfEst} has low bias]
\label{prop:low-bias}
    For any distribution $\ArbDist$ over $\bits^m$, coordinate $i \in [m]$, and $\eps > 0$,
    \begin{equation*}
        \left|\Ex\left[\InfEst(\ArbDist, i, \eps)\right] - \Inf_i(f_{\ArbDist})\right| \leq \eps
    \end{equation*}
    where $\InfEst$ is as defined in \Cref{fig:inf-est}.
\end{proposition}

Before proving \Cref{prop:low-bias}, we note that it implies a high accuracy estimator.

\begin{corollary}[Estimating influences]
\label{cor:high-accuracy}
    For any $\eps, \delta > 0$, there is an efficient algorithm that given unknown distribution $\ArbDist$, computes an estimate of $\Inf_i(f_{\ArbDist})$ to accuracy $\pm \eps$ with probability at least $1 - \delta$ using $O(\log(1/\delta)/ \eps^4)$ subcube conditional samples from $\ArbDist$.
\end{corollary}
\begin{proof}
    The algorithm outputs the mean of $O(\log(1/\delta)/\eps^2)$ independent calls to $\InfEst(\ArbDist, i , \eps/2)$. Each call to $\InfEst(\ArbDist, i , \eps/2)$ gives an output bounded within $[0,1]$. By Hoeffing's inequality, if $\textbf{est}$ is the mean of $O(\log(1/\delta)/\eps^2)$ independent calls to $\InfEst(\ArbDist, i , \eps/2)$,
    \begin{equation*}
        \Pr\left[\left|\textbf{est} -  \Ex\left[\InfEst(\ArbDist, i, \eps/2)\right] \right| \geq \eps/2\right] \leq \delta.
    \end{equation*}
    The result then follows from triangle inequality and \Cref{prop:low-bias}.
\end{proof}

We prove that $\InfEst$ has low bias.
\begin{proof}[Proof of \Cref{prop:low-bias}]
    For each $x \in \bits^n$, let
    \begin{equation*}
        p(x) \coloneqq \frac{\ArbDist(x)}{\ArbDist(x) + \ArbDist(x^{\oplus i})}
    \end{equation*}
    be the relative weight of $x$ in the subcube containing $x$ and $x^{\oplus i}$. Then, we can rewrite the influence as: 
    \begin{align*}
        \Inf_{i}(f_{\ArbDist}) &= \sum_{x \in \bits^m} \frac{1}{2^m}\left| f_{\ArbDist}(x) - f_{\ArbDist}(x^{\sim i}) \right |\\ 
         &= \sum_{x \in \bits^m} \left| \ArbDist(x) - \ArbDist(x^{\sim i}) \right | \tag{$f_{\ArbDist}(x) = 2^m \ArbDist(x)$}\\
        &= \frac{1}{2} \cdot \sum_{x \in \bits^m}\left| \ArbDist(x) - \ArbDist(x^{\oplus i}) \right| \tag{$x^{\sim i} = x^{\oplus i}$ wp $\frac{1}{2}$, and otherwise $x^{\sim i} = x$} \\
        &= \frac{1}{2} \cdot \sum_{x \in \bits^m}\left(\ArbDist(x) + \ArbDist(x^{\oplus i})\right) \cdot\left| p(x) - p(x^{\oplus i}) \right|. \tag{definition of $p(x)$} 
    \end{align*}
    Then, using the fact that $p(x) = 1 - p(x^{\oplus i})$, and distributing the $(\ArbDist(x) + \ArbDist(x^{\oplus i}))$ term, we can write
    \begin{align*}
        \Inf_{i}(f_{\ArbDist}) &= \frac{1}{2} \sum_{x \in \bits^m}\ArbDist(x) \cdot \left|p(x) - (1 - p(x)) \right| + \ArbDist(x^{\oplus i}) \cdot \left|p(x^{\oplus i }) - (1 - p(x^{\oplus i})) \right| \\
        &= \sum_{x \in \bits^m}\ArbDist(x) \cdot \left|p(x) - (1 - p(x)) \right|
    \end{align*}
    where, in the last step, we used the fact that summing over $x \in \bits^m$ is equivalent to summing over $x^{\oplus i} \in \bits^m$. Let $\hat{p}(x)$ be the random variable for the estimate of $p(x)$ computed by $\InfEst$ step \ref{step:est-p}. We bound the bias of \InfEst.
    \begin{align*}
         \Big|\Ex\left[\InfEst(\ArbDist, i, \eps)\right] &- \Inf_i(f_{\ArbDist})\Big| \\
         &= \left|\Ex_{\bx \sim \ArbDist}\left[\left|2\hat{p}(\bx) - 1\right|\right] - \Ex_{\bx \sim \ArbDist}\left[\left|2p(\bx) - 1\right|\right] \right| \tag{$2p-1 = p - (1-p)$} \\
         &= \left|\Ex_{\bx \sim \ArbDist}\left[\left|2\hat{p}(\bx) - 1\right| - \left|2p(\bx) - 1\right|\right] \right| \tag{linearity of expectation} \\
         &\leq \Ex_{\bx \sim \ArbDist}\left[\big|\left|2\hat{p}(\bx) - 1\right| - \left|2p(\bx) - 1\right|\big|\right]  \tag{Jensen's inequality} \\
         &\leq 2\left|\Ex_{\bx \sim \ArbDist}\left[\left|\hat{p}(\bx) - p(\bx)\right|\right]\right|. \tag{$\big||a| - |b|\big| \leq |a - b|$}
    \end{align*}
    For any $x \in \bits^m$, $\hat{p}(x)$ is the average of $\left\lceil1/\eps^2\right\rceil$ random variables each which is $1$ with probability $p(x)$ and $0$ otherwise. As a result, $\Ex[\hat{p}(x)] = p(x)$, and $\Var[\hat{p}(x)] \leq \eps^2/4$. Applying Jensen's inequality, we conclude
    \begin{equation*}
        \left|\Ex\left[\InfEst(\ArbDist, i, \eps)\right] - \Inf_i(f_{\ArbDist})\right| \leq 2\sqrt{\Var[\hat{p}(x)]} \leq \eps. \qedhere 
    \end{equation*}
\end{proof}
\section{Lifting uniform distribution learners: Proof of~\Cref{thm:lift}}

In this section, we show how to lift algorithms that learn over the uniform distribution to algorithms that learn with respect to arbitrary distributions, where the sample complexity and runtime scale with the decision tree complexity of the distribution. We first define our goal formally.

\begin{definition}[Learning with respect to a class of distributions]
    \label{def:learn}
    For any concept class $\mathscr{C}$ of functions $f: \bits^n \to \zo$, $\eps, \delta > 0$, set of distributions $\mathscr{D}$ with support $\bits^n,$ and $m, d \in \N$, we say that an algorithm $\mcA$ $(\eps,\delta)$-learns $\mathscr{C}$ for distributions $\mathscr{D}$ using $m$ samples if the following holds: For any $\mathcal{D} \in \mathscr{D}$ and any $f^\star \in \mathscr{C}$, given $m$ iid samples of the form $(\bx, f^\star(\bx))$ where $\bx \sim \mathcal{D}$, $\mcA$ outputs a hypothesis $h$ satisfying
    \begin{equation*}
        \Prx_{\bx \sim \mathcal{D}}[f^\star(\bx) \neq h(\bx)] \leq \eps.
    \end{equation*}
    with probability at least $1 - \delta$.
\end{definition}

Generally, we can think of $\delta$ as any fixed constant, as the success probability can always be boosted.
\begin{fact}[Boosting success probability]
\label{fact:boost}
    Given an algorithm $\mcA$ that $(\eps,\frac{1}{2})$-learns a concept $\mathscr{C}$ using $m$ samples, for any $\delta > 0$, we can construct an $\mcA'$ that $(1.1\eps,\delta)$-learns $\mathscr{C}$ using $m \cdot \poly(1/\eps, \log(1/\delta)) $ samples.
\end{fact}
\begin{proof}
    By repeating $\mcA$ $\log(1/\delta)$ times, we can guarantee that with probability at least $1 - \delta/2$, one of the returned hypotheses is $\eps$-close to $f^\star$. The accuracy of each of these hypothesis can be estimated to accuracy $\pm 0.05\eps$ using $O(\log(1/\delta)/\eps^2)$ random samples, and the most accurate one returned. With probability at least $1 - \delta$, that hypothesis will have at most $1.1\eps$ error.
\end{proof}

Our goal is to learn with respect to the class of all low-depth decision tree distributions. We use the following natural assumption on the concept class, which includes almost every concept class considered in the learning theory literature.

\begin{definition}[Closed under restriction]
    \label{def:closed-under-restriction}
    A concept class $\mathscr{C}$ of functions $f: \bits^n \to \zo$ is \emph{closed under restriction} if, for any $f \in \mathscr{C}$, $i \in [n]$, and $b \in \bits$, the restriction $f_{i = b}$ is also in $\mathscr{C}$.
\end{definition}

We will use \Cref{thm:decompose} to first decompose $\mcD$ into a mixture of nearly uniform distributions, and then run our learner on each of those distributions, as described in \Cref{fig:lift}.

\begin{figure}[h]

  \captionsetup{width=.9\linewidth}
\begin{tcolorbox}[colback = white,arc=1mm, boxrule=0.25mm]
\vspace{3pt} 

$\Lift(T, \mcA, S)$:  \vspace{6pt} \\
\textbf{Input:} A decision tree $T$, an algorithm for learning in the uniform distribution $\mcA$, and a random labeled sample $S$. \vspace{5pt} \\
\textbf{Output:} A hypothesis \vspace{4pt}

\ \ For each leaf $\ell \in T$ \{\vspace{-6pt}
\begin{enumerate}
    \item Let $S_\ell$ be the subset of points in $S$ that reach $\ell$.
    \item Create a set $S_{\ell}'$ consisting of points in $S_{\ell}$ but where all coordinates queried on the root-to-leaf path for $\ell$ are rerandomized independently (this makes the marginal over the input uniform).
    \item Use $\mcA$ to learn a hypothesis, $h_\ell$, with $S_{\ell}'$ as input.
\end{enumerate}
\vspace{-6pt}
\ \ \}\vspace{6pt}

\ \ Return the hypothesis that, when given an input $x$, first determines which leaf $\ell \in T$ that $x$ follows and then outputs $h_\ell(x)$.
\end{tcolorbox}
\caption{Pseudocode lifting a uniform distribution learner to one which succeeds on decision tree distributions. In this pseudocode, we assume that we have a decision tree representation which is close to the distribution, which can be accomplished using $\BuildDT$ in \Cref{fig:BuildDT}.}
\label{fig:lift}
\end{figure}

For our first result, we will assume that we already have a learner that succeeds on distributions that are sufficiently close to uniform.

\begin{definition}[Robust learners]
    For any concept class $\mathscr{C}$ of functions $f: \bits^n \to \zo$ and algorithm $\mcA$, we say that $\mcA$ $(\eps, \delta, c)$-\emph{robustly learns} $\mathscr{C}$ using $m$ samples under the uniform distribution if, for any $\eta > 0$ and the class of distributions
    \begin{equation*}
        \mathscr{D}_{\mathrm{TV, \eta}} \coloneqq \left\{\text{Distributions } \mathcal{D} \text{ over }\bits^n \text{ where } \TV(\mcU, \mathcal{D}) \leq \eta\right\},
    \end{equation*}
    $\mcA$ $(\eps + c\eta, \delta)$-learns $\mathscr{C}$ for the distributions in $\mathscr{D}_{\mathrm{TV, \eta}}$ using $m$ samples.
\end{definition}

The study of robust learners is part of a long and fruitful line of work. In particular, every learner that is robust to nasty noise \cite{BEK02} meets our definition of robust learners. 

Our result will also apply to learners that aren't explicitly robust. This is because \emph{every} learner is robust for $c = O(m)$.
\begin{proposition}
    \label{prop:auto-robust}
    For any concept class $\mathscr{C}$ and algorithm $\mcA$, is $\mcA$ $(\eps, \delta)$-learns $\mathscr{C}$ using $m$ samples under the uniform distribution, then $\mcA$ also $(\eps, \delta + \frac{1}{3}, 3m)$-robustly learns $\mathscr{C}$ using $m$ samples.
\end{proposition}
\begin{proof}
    Fix any $\eta > 0$.  Our goal is to show that $\mcA$ $(\eps + 3m \eta, \delta)$-learns $\mathscr{C}$ for distributions in $\mathscr{D}_{\TV, \eta}$. If $\eta \geq \frac{1}{3m}$, this is obviously true as any hypothesis has error $\leq 1$. We therefore need only consider $\eta < \frac{1}{3m}$. When $\mcA$ receives a sample from $\mcU^m$, it returns a hypothesis with error $\leq \eps$ with probability at least $1 - \eta$. Instead, $\mcA$ is receiving a sample from $\mcD^m$, where $\TV(\mcU, \mcD) < \frac{1}{3m}$. The success probability of any test given a sample from $\mcD^m$ rather than $\mcU^m$ can only differ by at most
    \[\TV(\mcU^m, \mcD^m) \leq m \cdot \TV(\mcU,\mcD) < \frac{1}{3}.\]
    Therefore, for $\eta < \frac{1}{3m}$, $\mcA$ succeeds wp at least $\delta + \frac{1}{3}$, as desired.
\end{proof}

We now state the main result of this section.

\begin{theorem}
    \label{thm:lift-given-DT}
    Choose any concept class $\mathscr{C}$ of functions $\bits^n \to \zo$ closed under restrictions, $\eps,\delta, c > 0$, $m,d \in \N$, and algorithm $\mcA$ that $(\eps, \delta/(2 \cdot 2^d), c)$-robustly learns $\mathscr{C}$ using $m$ samples under the uniform distribution. For any function $f^\star \in \mathscr{C}$, distribution $\mcD$ over $\bits^n$, depth-$d$ decision tree $T:\bits^n \to \R$ computing the PMF of a distribution $\mcD_T$ where
    \begin{equation*}
        \TV(\mcD, \mcD_T) \leq \frac{\eps}{c},
    \end{equation*}
     and sample size of
    \begin{equation*}
        M = m \cdot \poly\left(2^d, \frac{1}{\eps}, \log\left(\frac{1}{\delta}\right)\right).
    \end{equation*}
    Let $\bS$ a size-$M$ iid sample of labeled points $(\bx, f^\star(\bx))$ where $\bx \sim \mcD$. The output of $\Lift(T,\mcA, \bS)$ is $O(\eps)$-close to $f^\star$ w.r.t $\mcD$ with probability at least $1 - \delta$.
\end{theorem}
By \Cref{fact:boost}, an algorithm with constant failure probability could be transformed into one with the failure probability required by \Cref{thm:lift-given-DT} with only a $\poly(d, 1/\eps, \log(\delta))$ increase in the sample size. Therefore, would \Cref{thm:lift-given-DT} still holds when $\mcA$ $(\eps, \frac{1}{2}, c)$-robustly learns $\mathscr{C}$ if $\Lift$ applies \Cref{fact:boost} to boost the success probability of $\mcA$. Before proving \Cref{thm:lift-given-DT}, we show how it implies our main result.

\begin{theorem}[Lifting uniform-distribution learners, formal version of \Cref{thm:lift}]
\label{thm:lift-formal}
Choose any concept class $\mathscr{C}$ of functions $\bits^n \to \zo$ closed under restrictions, $\eps, c > 0$, $m,d \in \N$. If there is an efficient algorithm, $\mcA$, that $(\eps, \frac{1}{2})$-learns $\mathscr{C}$ using $m$ samples for the uniform distribution, then for $M = \poly(n) \cdot \left(\frac{dm}{\eps}\right)^{O(d)}$,
\begin{itemize}
    \item[$\circ$] There is an algorithm that $(\eps, \frac{1}{6})$-learns $\mathscr{C}$ using $M$ samples for monotone distributions representable by a depth-$d$ decision tree.
    \item[$\circ$] There is an algorithm that uses $M$ conditional subcube samples from $\mcD$ and $M$ random samples labeled by the target function that learns $\mathscr{C}$ to $\eps$-accuracy for arbitrary (not necessarily monotone) distributions representable by a depth-$d$ decision tree.
\end{itemize}
In both cases, the algorithm runs in time $\poly(n, M)$.
\end{theorem}
\begin{proof}[Proof of \Cref{thm:lift-formal} given \Cref{thm:lift-given-DT}]
    Using \Cref{thm:decompose}, we can learn the input distribution to TV-distance accuracy $\frac{\eps}{3m}$ with a decision tree hypothesis.  By \Cref{prop:auto-robust}, $\mcA$ $(\eps, \frac{1}{2}, 3m)$-robustly learns $\mathscr{C}$ for the uniform distribution, which can be boosted to failure probability $O(2^{-d})$ using \Cref{fact:boost}. Applying \Cref{thm:lift-given-DT} gives the desired result, with the desired runtime following from \Cref{prop:runtime-lift}.
\end{proof}

 The remainder of this section is devoted to the proof of \Cref{thm:lift-given-DT}. We'll use the following proposition.



\begin{proposition}
    \label{prop:each-leaf-many}
    For any distribution $\mcD$ over $\bits^n$, depth-$d$ decision tree $T$, $m \in \N$ and $p, \delta > 0$, as long as
    \begin{equation}
        \label{eq:M-each-leaf-many}
        M \geq O\left(\frac{d + m + \log(1/\delta)}{p}\right)
    \end{equation}
    for a random sample of $M$ points from $\mcD$, the probability there is a leaf $\ell \in T$ satisfying:
    \begin{enumerate}
        \item High weight: The probability a random sample from $\mcD$ reaches $\ell$ is at least $p$,
        \item Few samples: The number of points in the size-$M$ sample that reach $\ell$ is less than $m$
    \end{enumerate}
    is at most $\delta$.
\end{proposition}
\begin{proof}
    Fix a single leaf $\ell \in T$ with weight at least $p$. Then, the expected number of points that reach this leaf is $\mu \geq Mp$. As long as $\mu \geq 2m$, applying multiplicative Chernoff bounds,
    \begin{equation*}
        \Pr[\text{Fewer than $m$ points reach $\ell$}] \leq \exp\left(-\frac{\mu}{8} \right).
    \end{equation*}
    Union bounding over all leaves $2^d$, it is sufficient to choose an $M$ where
    \begin{equation*}
        2^d \exp\left(-\frac{\mu}{8} \right) \leq \delta.
    \end{equation*}
    This is satisfied for the $M$ from \Cref{eq:M-each-leaf-many}.
\end{proof}

We'll also use that if we have a decision tree $T$ that has learned the PMF to $\mcD$ to high accuracy, $\mcD$ restricted to the leaves of $T$ is, on average over the leaves, close to uniform.
\begin{fact}[Lemma B.4 of \cite{BLMT-boosting}]
    \label{prop:TV-distance-split}
    For any distribution $\mcD$ and decision tree $T$ computing the PMF for some distribution $\mcD_{T}$,
    \begin{equation*}
        \sum_{\text{leaves }\ell \in T} \Prx_{\bx \sim \mcD}[\bx\text{ reaches }\ell] \cdot \TV(\mcD_{\ell}, (\mcD_T)_\ell) \leq 2\cdot \TV(\mcD, \mcD_{T}).
    \end{equation*}
\end{fact}

\begin{proof}[Proof of \Cref{thm:lift-given-DT}]
    First, we split up the accuracy of $h = \Lift(T, \mcA, \bS)$ into the accuracy of the hypotheses $h_{\ell}$ learned at each leaf $\ell$: 
    \begin{equation*}
        \Prx_{\bx \sim \mcD}[h(\bx) \neq f^\star(\bx)] = \sum_{\ell \in T}\Prx_{\bx \sim \mcD}[\bx\text{ reaches }\ell] \cdot \Prx_{\bx \sim \mcD_\ell}[h_{\ell}(\bx) \neq f^\star(\bx)].
    \end{equation*}
    Each hypothesis $h_{\ell}$ is learned by running $\mcA$ on the sample $\bS_{\ell}'$. First, we argue that for all leaves $\ell$ with much of $\mcD$'s weight, will, whp, have $\geq m$ samples. Applying \Cref{prop:each-leaf-many} with probability at least $1 - \delta/2$, for all $\ell \in T$ with $\Prx_{\bx \sim \mcD}[\bx\text{ reaches }\ell] \geq  \frac{\eps}{2^d}$, the size of $\bS_{\ell}'$ is at least $m$. Conditioning on that being true, we have that,
    \begin{align*}
        \Prx_{\bx \sim \mcD}[h(\bx) \neq f^\star(\bx)] &\leq  \sum_{\ell \in T}\Prx_{\bx \sim \mcD}[\bx\text{ reaches }\ell] \cdot \Prx_{\bx \sim \mcD_\ell}[h_{\ell}(\bx) \neq f^\star(\bx) \mid |\bS_{\ell}'| \geq m] \\
        &\quad+  \sum_{\ell \in T}\Prx_{\bx \sim \mcD}[\bx\text{ reaches }\ell] \cdot \Ind[\Prx_{\bx \sim \mcD}[\bx\text{ reaches }\ell] \leq p] \\
        & \leq \sum_{\ell \in T}\Prx_{\bx \sim \mcD}[\bx\text{ reaches }\ell] \cdot \Prx_{\bx \sim \mcD_\ell}[h_{\ell}(\bx) \neq f^\star(\bx) \mid |\bS_{\ell}'| \geq m] + \eps,
    \end{align*}
    where the last line uses that $T$ has at most $2^d$ leaves and $p = \frac{\eps}{2^d}$.

    Each point in $\bS_{\ell}'$ is an iid sample with the input uniform over $\bits^n$ and labeled by the function $f^\star_{\ell}$. As $\mathscr{C}$ is closed under restrictions, $f^\star_{\ell} \in \mathscr{C}$. Therefore,
    \begin{equation*}
        \Prx_{\bS}\left[\left[\Prx_{\bx \sim \mcD_\ell}[h_{\ell}(\bx) \neq f^\star(\bx)\mid |\bS_{\ell}'| \geq m ]\right] \leq \eps +  c \cdot \TV(\mcD_{\ell}, \mcU)\right] \geq 1 - \frac{\delta}{2 \cdot 2^d}.
    \end{equation*}
    We union bound over all $2^d$ leaves $\ell$ and the earlier event that $| \bS_{\ell}| \geq m$ whenever $\Prx_{\bx \sim \mcD}[\bx\text{ reaches }\ell] \geq p$, we have with probability at least $1 - \delta$,
    \begin{align*}
        \Prx_{\bx \sim \mcD}[h(\bx) \neq f^\star(\bx)] & \leq \sum_{\ell \in T}\Prx_{\bx \sim \mcD}[\bx\text{ reaches }\ell]\cdot \left(2\eps + c \cdot \TV(\mcD_{\ell}, \mcU)\right)\\
        &= 2\eps + c\cdot \sum_{\ell \in T}\Prx_{\bx \sim \mcD}[\bx\text{ reaches }\ell]\cdot \TV(\mcD_{\ell}, \mcU)\\
        &= 2\eps + c\cdot \sum_{\ell \in T}\Prx_{\bx \sim \mcD}[\bx\text{ reaches }\ell]\cdot \TV(\mcD_{\ell}, (\mcD_T)_\ell)\\
        &\leq  2\eps + c\cdot 2 \cdot \TV(\mcD, \mcD_T) \tag{\Cref{prop:TV-distance-split}} \\
        & \leq 4\eps. \tag{$\TV(\mcD, \mcD_T) \leq \frac{\eps}{c}$}
    \end{align*}
    As desired, we have $\Lift$ learns an $O(\eps)$-accurate hypothesis w.h.p.
\end{proof}

\begin{proposition}[Runtime of \Lift]
    \label{prop:runtime-lift}
    Assuming unit-time calls to $\mathcal{A}$, given a depth-$d$ decision tree $T$, $\Lift(T, \mathcal{A}, S)$ runs in time $O(n2^d \cdot |S|)$.
\end{proposition}
\begin{proof}
    The number of leaves in $T$ is at most $2^d$. Since each point in $S$ is in $\bits^n$, the entire sample takes $O(n \cdot |S|)$ bits to represent. For each leaf $\ell$, it takes $O(n \cdot |S|)$ time to loop through this representation, find the points consistent with $\ell$, and create the modified dataset $S_\ell'$, and pass it into $\mcA$. Repeating this over all leaves can be done in $O(n2^d \cdot |S|)$ time.
\end{proof}

\begin{remark}[The agnostic setting]
\label{remark:agnostic}
    A popular variant of learning, as defined in \Cref{def:learn}, is the \emph{agnostic setting}. In this generalization of standard learning, rather than assume $f^\star \in \mathscr{C}$, $\mcA$ is required to output a hypothesis with error $\opt + \eps$, where $\opt$ is the minimum error of a hypothesis in $\mathscr{C}$ w.r.t $f^\star$. It's straightforward to see that \Cref{thm:lift-given-DT}, and therefore \cref{thm:lift-formal}, extend to the agnostic setting, where if the uniform distribution learning $\mcA$ succeeds in the agnostic setting, that learner is upgraded to one that succeeds for decision tree distributions in the agnostic setting. This is because, if there is an $f \in \mathscr{C}$ with error $\opt$ w.r.t. $f^\star$, then the average error of $f$ over the leaves of a tree $T$ w.r.t $f^\star$ will also be at most $\opt$.
\end{remark}

\section*{Acknowledgements}

We thank the STOC reviewers for their detailed feedback, especially for the references to the literature on semi-supervised learning. 

Guy and Li-Yang are supported by NSF awards 1942123, 2211237, and 2224246. Jane is
supported by NSF Award CCF-2006664. Ali is supported by a graduate fellowship award from Knight-Hennessy Scholars at Stanford University.

\bibliographystyle{alpha}
\bibliography{ref}

\end{document}